\documentclass[letterpaper, 10 pt, conference]{ieeeconf}  

\IEEEoverridecommandlockouts                              

\usepackage{booktabs,xmpmulti,amsmath,amssymb,mathtools} 
\usepackage{transparent}
\usepackage{xcolor}
\usepackage{cite}
\usepackage[inkscapearea=page, inkscapearea=nocrop, inkscape=true]{svg}

\usepackage{tikz}
\usetikzlibrary{patterns}
\usepackage{pgfplots}
\tikzstyle{bag} = [align=left]
\usetikzlibrary{shapes.misc}

\tikzset{cross/.style={cross out, draw=black, fill=none, minimum size=2*(#1-\pgflinewidth), inner sep=0pt, outer sep=0pt}, cross/.default={2pt}}

\usepackage{bbm}

\newcommand{\nfloor}[1]{\ensuremath{\lfloor #1 \rfloor}}
\newcommand{\nceil}[1]{\ensuremath{\lceil #1 \rceil}}

\newcommand{\st}{\text{ subject to }}

\DeclareMathOperator*{\minimize}{\ensuremath{\mathrm{minimize}}}

\title{\LARGE \bf
Bridging conformal prediction and scenario optimization
}

\newtheorem{sassumption}{Standing Assumption}
\newtheorem{assumption}{Assumption}
\newtheorem{definition}{Definition}
\newtheorem{thm}{Theorem}
\newtheorem{proposition}{Proposition}

\newtheorem{remark}{Remark}
\newtheorem{corollary}{Corollary}

\author
{
        Niall O'Sullivan$^{1}$, Licio Romao$^{2}$, and Kostas Margellos$^{1}$
        \thanks{
                $^{1}$ Niall O'Sullivan and Kostas Margellos are with the Department of Engineering Science, University of Oxford, Oxford OX1 3PJ, United Kingdom. Email: {\tt\small niall.osullivan@worc.ox.ac.uk; kostas.margellos@eng.ox.ac.uk}
        }
        \thanks{
                $^{2}$ Licio Romao is with the Department of Wind and Energy Systems, Technical University of Denmark, 2800 Kgs. Lyngby, Denmark. Email: {\tt\small licio@dtu.dk}
        }%
        \thanks{For the purpose of Open Access, the authors have applied a CC BY copyright license to any Author Accepted Manuscript (AAM) version arising from this submission.}
}

\begin{document}

\maketitle
\thispagestyle{empty}
\pagestyle{empty}


\begin{abstract}

        Conformal prediction and scenario optimization constitute two important classes of statistical learning frameworks to certify decisions made using data. They have found numerous applications in control theory, machine learning and robotics. Despite intense research in both areas, and apparently similar results, a clear connection between these two frameworks has not been established. By focusing on the so-called vanilla conformal prediction, we show rigorously how to choose appropriate score functions and set predictor map to recover well-known bounds on the probability of constraint violation associated with scenario programs. We also show how to treat ranking of nonconformity scores as a one-dimensional scenario program with discarded constraints, and use such connection to recover vanilla conformal prediction guarantees on the validity of the set predictor. We also capitalize on the main developments of the scenario approach, and show how we could analyze calibration conditional conformal prediction under this lens. Our results establish a theoretical bridge between conformal prediction and scenario optimization.  

\end{abstract}

\section{Introduction}

Modern operations in critical infrastructures, such as transportation, power systems, and robotics, rely on past data to make informed decisions on long-term planning and real-time operation strategies. In the context of power systems, flexibility aggregators make decisions using available data to manage their energy resources, while providing services to the electricity grid, e.g., see \cite{Dimitra2022,Oren2016}. Similar high-stake decisions are made in the context of autonomous systems, where robots must make decisions based on available data to navigate in unknown environments while safely achieving their goals; see \cite{Lindemann2024} and references therein. 

Conformal prediction and the scenario approach theory constitute widely employed techniques that enable a user to rigorously certify the underlying decision-making process. The former, originally proposed in seminal papers \cite{Vovk1999,Vovk2005,Vovk2013}, has gained prominence within the control community due to its simplicity and strong guarantees. It has been recently applied to a wide range of applications, including drug discovery, robotic motion planning, and several machine learning frameworks \cite{Angelopoulos2023,Carlevaro2024,Lindemann2024}. Main ingredients of conformal prediction include the set predictor, which is an output set that contains a new sample with a probability that can be quantified in a distribution-free manner; and the nonconformity scores, which are used to assess the quality of the decision and inform the construction of the set predictor. In the conformal prediction setting, the resulting predictor set is accompanied by distribution-free guarantees on its validity \cite{Vovk2012,Vovk2013,Vovk2013a,Angelopoulos2023}, where validity relates to the property of the set predictor contains a pre-specified mass of all possible samples. 

The so-called \emph{scenario approach} has been developed in a series of seminal papers \cite{Calafiore2004,Calafiore2006,Campi2008,Campi2010,Campi2018} and has received wide attention within the control community due to its ability to attach violation certificates to the solution of an optimization problem. Similarly to conformal prediction, the scenario approach theory has been applied to a wide range of applications, including machine learning \cite{Campi2021}, feedback control design \cite{Calafiore2006, Calafiore2011, Badings2023, Lin2024,Foffano2023}, and uncertainty quantification \cite{Dabbene2022}. The scenario approach involves solving an optimization problem with constraints corresponding to different samples/realizations of some uncertain parameter. A fundamental result has been the (exact) characterization of the  cumulative distribution and the expected value of the so called probability of constraint violation, i.e., the probability that the optimal solution violates the constraint corresponding to a new unseen sample. Such developments serve as generalization guarantees to solutions computed on the basis of a finite number of samples. Recent advancements involve tighter guarantees on scenario optimization with sample discarding \cite{Romao2023,Romao2023a}, connections to statistical learning based on the notion of compression \cite{Campi2023,Margellos2015}, and \emph{a posteriori} assessements \cite{Garatti2022}.

\begin{figure}[h!]
\centering
\begin{tikzpicture}
\draw [draw=black] (-6.5,-0.2) rectangle (-3.5,1.3);
\filldraw [fill=black!20, draw=black] (-6.5,-0.2) rectangle (-3.5,1.3);
\draw[] (-5,1.2) node[below, bag]{$\text{Conformal}$};
\draw[] (-5,0.8) node[below, bag]{$\text{Prediction}$};
\draw[] (-5,0.4) node[below, bag]{$\text{(Vanilla)}$};
\draw [draw=black] (-1.5,-0.2) rectangle (1.5,1.3);
\filldraw [fill=black!20, draw=black] (-1.5,-0.2) rectangle (1.5,1.3);
\draw[] (0,1) node[below, bag]{$\text{Scenario}$};
\draw[] (0,0.6) node[below, bag]{$\text{Optimization}$};
\draw [draw=black] (-4.5,-2.5) rectangle (-0.5,-1);
\filldraw [fill=black!20, draw=black] (-4.5,-2.5) rectangle (-0.5,-1);
\draw[] (-2.5,-1.1) node[below, bag]{$\text{Conformal}$};
\draw[] (-2.5,-1.5) node[below, bag]{$\text{Prediction}$};
\draw[] (-2.5,-1.9) node[below, bag]{$\text{(Calibration Conditional)}$};
\draw[->,ultra thick] (-3.5,0.9)--(-1.5,0.9) node[left]{};
\draw[<-,ultra thick] (-3.5,0.2)--(-1.5,0.2) node[left]{};
\draw[->,ultra thick] (0,-0.2)--(-2.5,-1) node[left]{};
\draw[] (-2.5,1.4) node[below, bag]{$\text{Theorem 4}$};
\draw[] (-2.5,0.7) node[below, bag]{$\text{Theorem 5}$};
\draw[] (0,-0.4) node[below, bag]{$\text{Theorem 6}$};
\end{tikzpicture}
\caption{Connections between conformal prediction and scenario optimization. From conformal prediction to scenario approach: By appropriate choice of nonconformity scores and a set predictor in vanilla conformal prediction, we recover well-known results on the expected value of the probability of constraint violation in scenario programs (Theorem \ref{thm:CP-SP-connection-average-violation}). From scenario approach to conformal prediction: By interpreting the ranking of nonconformity scores as an one-dimensional scenario program with discarded constraints, we derive vanilla conformal prediction results (Theorem \ref{thm:scen_conf_van}). We also view the calibration conditional conformal prediction under the scenario approach lens, using the results on the cumulative distribution of the probability of constraint violation in scenario programs (Theorem \ref{thm:scen_conf_cond}).}
\label{fig:CP-versus-SA-contributions}
\end{figure}
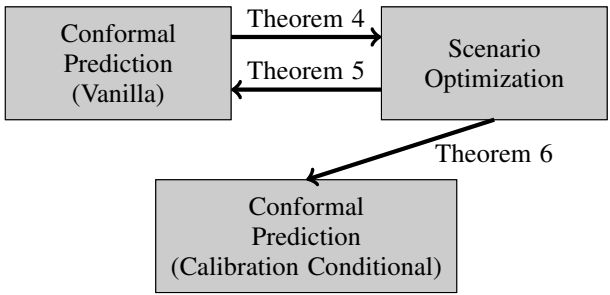

In this paper we aim at establishing formal relationships between the two approaches, that we believe will open the road for transferring results from one domain to the other. Our main contributions are:
\begin{enumerate}
\item From conformal prediction to scenario approach: We show how to choose an appropriate score function and set predictor, at the core of conformal prediction, to bound the expected probability of constraint violation in scenario optimization (Theorem 4). 
\item From scenario approach to vanilla conformal prediction: We show that using a specific scenario optimization program together with a discarding mechanism we can view the so called vanilla conformal prediction under a scenario approach lens and establish the corresponding vanilla conformal prediction bounds (Theorem 5). To achieve this, we build on results related to the expected value of the probability of constraint violation in scenario optimization. \\
\item From scenario approach to calibration conditional conformal prediction: We capitalize on the characterization of the cumulative distribution (rather than the expected value) of the probability of constraint violation in scenario optimization and show how this can be employed to infer calibration conditional conformal prediction bounds (Theorem 6). 
\end{enumerate}

Besides attempts to contrast conformal prediction and the scenario applications in specific applications \cite{Coppola_2024}, \cite{Lin2024}, to the best of our knowledge this is the first formal connection between these domains. Our main contributions are also pictorially illustrated in Figure \ref{fig:CP-versus-SA-contributions}.


The remainder of the paper is organized as follows. Section \ref{sec:prelim} introduces the main concepts from conformal prediction and the scenario approach that are relevant to our developments, and provides some preliminary results. Section \ref{sec:conf_scen} shows how tools from conformal prediction can be leveraged to obtain certain results of the scenario approach theory. In Section \ref{sec:scen_conf} we view conformal prediction under a scenario approach lens. We first show how results on the characterization of the probability of constraint violation in scenario optimization can be leveraged to reinterpret the vanilla conformal prediction bounds (Section \ref{sec:scen_conf_van}), and then recover the so called calibration conditional conformal prediction results (Section \ref{sec:scen_conf_cond}).
Section \ref{sec:conc} concludes the paper and provides directions for future research. 

\subsection*{Notation}

We denote by $\mathbb{N}$ the set of natural numbers. We let $\Omega$ be an abstract probability space and call $(\Omega, \mathcal{F}, \mathbb{P})$ a probability space, where $\mathcal{F}$ is a $\sigma$-algebra of subsets of $\Omega$ and $\mathbb{P}$ is a probability measure defined on $\mathcal{F}$. 

\section{Background and preliminary results}  \label{sec:prelim}


\subsection{Conformal prediction}  \label{sec:prelim_conf}

Conformal prediction has gained significant attention due to its ability to ensure coverage guarantees for the output of general decision making schemes. It provides a distribution-free bound on the probability that a new realization does not belong to a generated prediction interval that is in turn based on a finite number of scenarios/data. Such a predictive framework relies on the construction of non-conformative scores to construct an output set that possesses rigorous guarantees on the probability that a new sample is not contained in the output set. More formally, the conformal prediction set-up can be described as follows \cite{Balasubramanian2014}.

For a given $m \in \mathbb{N}$, a \textit{set predictor} $\Gamma: \Omega^{m} \mapsto 2^{\Omega}$ is a set-valued map sending available data to a collection of realizations that are used to inform the prediction of the unknown outcome. The following measurability assumption is commonly enforced in the context of conformal prediction. 

\begin{sassumption}
        For any $m \in \mathbb{N}$, the subset
        \begin{equation}
                \left\{ (\omega_1, \ldots, \omega_m, \omega) \in \Omega^{m+1}: \omega \in \Gamma(\omega_1,\ldots, \omega_m)
                \right\},
        \end{equation}
        is a Borel measurable subset.
        \label{assump:CP-predictor-set-regularity}
\end{sassumption}

\textit{Validity} in the context of conformal prediction quantifies formally the size of the set 
\begin{equation}
        \big \{ (\omega_1, \ldots, \omega_m, \omega): \omega \notin \Gamma(\omega_1, \ldots, \omega_m) \big \},
        \label{eq:CP-validity}
\end{equation}
which is measurable by Assumption \ref{assump:CP-predictor-set-regularity}. We have the following crucial notion.

\begin{definition}
        For any distribution $\mathbb{P}^{m+1}$ defined on the space $\Omega^{m+1}$ and any $\delta \in (0,1)$, a set predictor is said to be \textit{conservatively valid} at significance level $\delta$ if $\mathbb{P}^{m+1}\big \{ (\omega_1, \ldots, \omega_m, \omega): \omega \notin \Gamma(\omega_1, \ldots, \omega_m) \big \} \leq \delta,$ where the argument of $\mathbb{P}^{m+1}$ is as defined in \eqref{eq:CP-validity}.
        \label{CP-validity}
\end{definition}

Conformal prediction aims at providing nontrivial conservatively valid set predictors. A natural way to construct conservatively valid set predictors is through the concepts of \textit{nonconformity measure and scores}, which effectively provide a scalarisation related to the quality of the prediction. 

Under the terminology used in the conformal prediction literature, a \textit{nonconformity measure} is a permutation-invariant measurable mapping $\mathcal{A}: \Omega^m \mapsto \mathbb{R}^m$ from the sample space $\Omega^m$ to the $m$-dimensional Euclidean space, that is, for all $\pi$ in the permutation group of $\{1, \ldots, m\}$ we have that

\begin{equation}
        \mathcal{A}(\omega_1, \ldots, \omega_m) = \mathcal{A}(\omega_{\pi(1)}, \ldots, \omega_{\pi(m)})),
        \label{eq:CF-nonconformity-score}
\end{equation}
Intuitively, permutation invariance enforces the nonconformity score to be a property of the samples themselves, and not of the order in which these are presented. 

Given a collection of samples $S = \{\omega_1,\ldots,\omega_m\} \subset \Omega^m$ and an additional ``test'' sample $\omega$, let $\mathcal{A}(S \cup \{\omega\}) = [R_1,\ldots,R_m,R]^\top \in \mathbb{R}^{m+1}$ be the output of the nonconformity measure when its input is the collection of the $m$ samples in $S$ and the additional sample $\omega$. The elements $R_i$, $i=1,\ldots,m$, can be thought of as nonconformity scores in the sense of  \cite{Balasubramanian2014,Lindemann2024}.
Notice the slight abuse of notation, where the argument of $\mathcal{A}$ lies in $\Omega^{m+1}$, while the output is a vector in $\mathbb{R}^{m+1}$. 
For simplicity of the exposition we impose the following assumption.
\begin{assumption} \label{ass:order}
Assume that for any collection $(\omega_1,\ldots,\omega_m,\omega) \in \Omega^{m+1}$ of independent samples from a distribution $\mathbb{P}$, the elements $R_1,\ldots,R_m$ are also independent, there are no repetitions, and are sorted in non-decreasing order.
\end{assumption}
The requirement that there are no repetitions and that the nonconformity scores are ordered could be relaxed for some of the subsequent statements, but we impose it here for simplicity. Also, for some of the presented results independence could be relaxed to the somewhat weaker notion of exchangability.

Define then the function
\begin{equation}
        f(\omega; S) = \frac{|\{i\in \{1, \ldots, m+1\}\}: R_i \geq R |}{m+1}, \label{eq:func_f}
\end{equation}
where $|\cdot|$ denotes the cardinality of its argument. 
For any fixed $\delta \in (0,1)$, the set predictor using \eqref{eq:CF-nonconformity-score} can be constructed by means of the rule
\begin{equation}
        \Gamma_\delta(\omega_1, \ldots, \omega_m)= \{ \omega \in \Omega: f(\omega; S) > \delta\}.
        \label{eq:Vanilla-CP}
\end{equation}
We then have an important result for the conformal predictor given in \eqref{eq:Vanilla-CP}, which will be referred to as the \textit{vanilla conformal predictor}.

\begin{thm}[Proposition 1.2 \cite{Balasubramanian2014}]
        Let $\mathcal{A}$ be a nonconformity measure, and suppose Assumption \ref{ass:order} holds. Fix $\delta \in (0,1)$.
        Then, the set predictor $\Gamma_\delta$ given in \eqref{eq:Vanilla-CP} is conservatively valid at significance level $\delta$, i.e.,
        \begin{equation}
                \mathbb{P}^{m+1}\big \{(\omega_1,\ldots,\omega_m,\omega) \in \Omega^{m+1}: \omega \notin \Gamma_\delta(S)\big \} \leq \delta.
        \end{equation}
        \label{thm:Vanilla-CP}
        \hfill $\square$
\end{thm}
Notice that the level of conservatism, i.e., the quality of the predictor in Theorem \ref{thm:Vanilla-CP}, relies on the quantity in \eqref{eq:func_f}, which in turn based on the choice of nonconformity measure.

We now discuss a direct byproduct of Theorem \ref{thm:Vanilla-CP}. Recall that  
$\mathcal{A}(S \cup \{\omega\}) = [R_1,\ldots,R_m,R]^\top \in \mathbb{R}^{m+1}$ includes the nonconformity scores, and notice that due to the definition of $f(\omega;S)$, we have that
\begin{align}
&\big \{ \omega \notin \Gamma_\delta(S) \big \} = \nonumber \\
&\big \{ \omega \in \Omega:~ \Big |\{i\in \{1, \ldots, m+1\}\}: R_i \geq R \Big | \leq \nfloor{\delta(m+1)} \big \}, \label{eq:sets}
\end{align} 
where $\nfloor{\cdot}$ denotes the highest integer lower than or equal to its argument.
Let $p = \nceil{(1-\delta)(m+1)}$ with $\nceil{\cdot}$ denoting the smallest integer higher than or equal to its argument, and consider the set
\begin{align}
Q:= \big \{ \omega \in \Omega:~ R> R_p \big \}.
\end{align}
Notice that due to the independence part of Assumption \ref{ass:order}, $R$ depends on $\omega$, and $R_i$
depends on $\omega_i$, for each $i=1,\ldots,m$. However, $R_p$ depends on $\omega_1,\ldots,\omega_m$. 
Under the choice of $p$, $R_p$ is the $1-\delta$ quantile of the empirical distribution over the discrete set $R_1,\ldots,R_m$ and $\infty$, where infinity is appended as a correction to obtain a finite sample truncation. Therefore, $R_p := \mathrm{Quantile}_{1-\delta}(R_1,\ldots,R_m,\infty)$, which implies that there exist $\nceil{(1-\delta)(m+1)}$ nonconformity scores lower than or equal to $R_p$.

Consider any $\omega \in Q$. There exist $\nceil{(1-\delta)(m+1)}$ nonconformity scores lower than or equal to $R_p$, and hence of $R$, since $R_p < R$ for any such $\omega$. Therefore, the number of nonconformity scores such that $R_i\geq R$ would be at most $m+1 - \nceil{(1-\delta)(m+1)} = \nfloor{\delta(m+1)}$, i.e., 
\begin{align}
\Big |\{i\in \{1, \ldots, m+1\}\}: R_i \geq R \Big | \leq \nfloor{\delta(m+1)} .
\end{align}
By \eqref{eq:sets}, any such $\omega$ would be such that $\omega \notin \Gamma_\delta(S)$.
We have thus proven the following corollary of Theorem \ref{thm:Vanilla-CP}.

\begin{corollary} \label{cor:viol}
Consider the setting of Theorem \ref{thm:Vanilla-CP}. Fix $\delta \in (0,1)$ and let  $p = \nceil{(1-\delta)(m+1)}$. We then have that 
\begin{align}
                \mathbb{P}^{m+1} \big \{(\omega_1,\ldots,\omega_m,\omega) \in \Omega^{m+1}: R> R_p \big \} \leq \delta.
\end{align}
       \label{cor:CP-stat-control-formulation} 
       \hfill $\square$
\end{corollary}
Corollary \ref{cor:viol} bounds the probability that a new ``test'' $R$ exceeds the $p$-th quantile over a discrete distribution based on $R_1,\ldots,R_m$; these samples are often referred to as the calibration dataset.

\subsection{Scenario approach}


The scenario approach is formulated in \cite{Campi2018} as natural methodology to deal with uncertainty in an optimization context. To this end, let $S = \{\omega_1, \ldots, \omega_m\}$ be a collection of $m$ independent samples from $\mathbb{P}$, which could encode different realizations of an uncertain parameter. Scenario optimization involves minimizing some objective function over a family of constraints, each of them corresponding to the associated sample $\omega_i$, $i=1,\ldots,m$. This gives rise to the so called \emph{scenario program}, which can be formalized as

\begin{align}
        \minimize_{x \in \mathcal{X}} & \quad c^\top x \nonumber \\
        \st & \quad g(x, \omega) \leq 0, \quad \forall \omega \in S \setminus \mathcal{R},
        \label{eq:scenario-program}
\end{align}
where $x \in \mathbb{R}^d$ is the optimization variable,  $\mathcal{X} \subset \mathbb{R}^d$ is an implicit domain, and $g: \mathbb{R}^d \times \mathbb{R}^q \mapsto \mathbb{R}$ is a constraint function which we assume is convex with respect to its first argument and measurable with respect to the second. It should be noted that the objective function is linear without loss of generality; any convex objective function could be admissible, and in such cases we could bring the optimization
problem in the standard scenario program format by means of an epigraphic reformulation.

We denote by $\mathcal{R}\subset S$ with $|R| = r$ a collection of $r$ samples that are removed from the set $S$. The introduction of $\mathcal{R}$ in the scenario program adds more flexibility as it allows removing samples considered as outliers (see results on sampling and discarding \cite{Campi2010,Romao2023}), and at the same time will play a crucial role in the next sections to reinterpret conformal prediction within the realm of scenario optimization.

We impose the following assumption.
\begin{assumption} \label{ass:unique}
        With $\mathbb{P}^{m}$-probability one with respect to the choice of $(\omega_1,\ldots,\omega_m)$, the (convex) scenario program \eqref{eq:scenario-program} has an non-empty feasibility region, and its optimal solution is unique.     
        \label{assump:scenario-program-existence}
\end{assumption}
This is rather standard in the scenario approach literature. It can be relaxed to allow for infeasible problem instances, while in case of multiple optimizers a particular one can be singled-out by means of a (deterministic) convex tie-break rule \cite{Calafiore2006}. 

We denote the (unique under Assumption \ref{ass:unique}) optimal solution of \eqref{eq:scenario-program} by $x^\star(S)$; notice that $x^\star(S):~\Omega^m \to \mathbb{R}^d$. Due to its dependency on the samples in $S$, this is a 
a random variable defined in the probability space $(\Omega^m, \mathcal{F}, \mathbb{P}^{m})$, where the $\sigma$-algebra is the product $\sigma$-algebra \cite{Salamon2016}. Notice that $x^\star(S)$ depends on all samples in $S$ even if constraints are enforced only on those in $S \setminus \mathcal{R}$. This is the case since to decide on the samples that are to be discarded, all samples in $S$ need to be inspected, so formally $\mathcal{R}$ depends on $S$ as well.

 Main developments in the scenario approach \cite{Calafiore2004,Campi2008,Campi2010} entail probably approximately correct (PAC) bounds on the probability of violation associated with the optimal solution of \eqref{eq:scenario-program}. In other words, consider the random variable $V:\Omega^m \mapsto [0,1]$ defined as
\begin{equation}
        V(S) = \mathbb{P}\{\omega \in \Omega: g(x^\star(S), \omega) >  0 \},
        \label{eq:violation-probability}
\end{equation}
which represents the violation probability associated with the optimal solution of the scenario program \eqref{eq:scenario-program}, when it comes to a new sample $\omega$. The scenario approach theory provides bounds on the cumulative distribution of $V(S)$ and its expected value;  see \cite{Campi2018} for more details. For a specific class of scenario programs termed \emph{fully-supported}, these bounds are also tight. We provide a formal definition of such programs below. Note that the scenario approach results are by no means restricted to this class of problems. However, we show in Section \ref{sec:scen_conf} that the optimization programs for which we will invoke the scenario approach exhibit this property, hence to streamline our exposition we limit our discussion to these instances.


\begin{definition}[Fully-supported scenario programs]
Consider a scenario program as in \eqref{eq:scenario-program} with $\mathcal{R} = \emptyset$.
       Such a programis said to be fully-supported if, with $\mathbb{P}^m$-probability one, the cardinality of the set
        \begin{equation}
              \big  \{ \omega \in S: x^\star(S) \neq x^\star(S \setminus \{\omega\}) \big \}
                \label{eq:fully-supported-scenario-program}
        \end{equation}
        is equal to $d$, where $x^\star(S\setminus \{\omega\})$ denotes the optimal solution of the scenario program \eqref{eq:scenario-program} with the sample $\omega$ removed.
        \label{def:fully-supported-scenario-program} 
\end{definition}

The collection of samples defined in \eqref{eq:fully-supported-scenario-program} is called the \textit{support set}, and consists of samples that if removed from the set $S$ will lead to a change in the optimal solution of the scenario program \eqref{eq:scenario-program}. Such a set can be thought of a sample compression in the sense of \cite{Margellos2015}, as solving the problem only with the these samples would yield the same solution with the one that would have been obtained had all the samples be employed.

We provide a fundamental result in the scenario approach theory that characterizes the cumulative distribution function of the probability of constraint violation. This result was first established in \cite{Campi2008} for any convex program and without any sample discarding, i.e., when $\mathcal{R} = \emptyset$. Here, we provide a variant of this result for the case of fully-supported scenario programs, and when $\mathcal{R} \neq \emptyset$.

\begin{thm}[Theorem 4, \cite{Romao2023}]
        Consider Assumption \ref{assump:scenario-program-existence}. Further assume that \eqref{eq:scenario-program} is fully-supported, and that the discarded samples are removed in an integer multiple of $d$ from the set $S$ by solving a cascade of scenario programs where the support set is discarded at each stage. We then have that  
        \begin{align}
                \mathbb{P}^m \big \{ S \in \Omega^m&: V(S) \leq \epsilon \big \} \nonumber \\
                &\geq 1 - \sum_{i = 0}^{r+d-1} \binom{m}{i} \epsilon^i (1-\epsilon)^{m-i}. 
        \end{align}
        \label{thm:scenario-approach-discarded-constraints}
        \hfill $\square$
\end{thm}
The discarding procedure of Theorem \ref{thm:scenario-approach-discarded-constraints} was introduced in \cite{Romao2023, Romao2023a}. Informally, it requires removing samples in batches of $d$ samples at a time, where the samples that are removed each time are the support ones, hence removing them guarantees that the solution changes; see \eqref{eq:fully-supported-scenario-program}.
The main result of Theorem \ref{thm:scenario-approach-discarded-constraints} would also be tight, i.e., it would hold with equality if all removed samples violate the optimal solution of the scenario program \eqref{eq:scenario-program} and also all interim solutions generated via the iterative discarding procedure (see Theorem 5 in \cite{Romao2023}).
 For the scenario programs that this theorem will be invoked in the sequel, this property is satisfied.
Note that this requirement is related to non-degeneracy assumptions in the scenario approach theory \cite{Campi2008}, as well as in some works on conformal prediction \cite{Vovk1999,Balasubramanian2014}.

Theorem \ref{thm:scenario-approach-discarded-constraints}, provides a characterization of the cumulative distribution of $V(S)$, namely, it bounds $\mathbb{P}^m \big \{ S \in \Omega^m: V(S) \leq \epsilon \big \}$ by the tail of the distribution of a binomial random variable (beta distribution). As such, we can exploit this to obtain bounds also for other moments, like the expected value. We summarize this in the theorem below.

\begin{thm}\label{thm:scenario-approach-expectation-violation}
Consider the same setting with Theorem \ref{thm:scenario-approach-discarded-constraints}. We then have that
\begin{align}
                \mathbb{E}_{S \sim \mathbb{P}^m} \big \{V(S) \big \} \leq \frac{r+d}{m+1}.
\end{align}
\hfill $\square$
\end{thm}
\begin{proof}
Theorem \ref{thm:scenario-approach-discarded-constraints} implies that cumulative distribution of $V(S)$ is lower-bounded by the tail of a binomial distribution. As a result, the associated density function would be concentrated to lower values compared to that of the binomial tail, which we denote by $q(\epsilon)$. This then implies that the expected value $V(S)$ will be upper-bounded by the expected value associated with the binomial tail, i.e.,
\begin{equation}
 \mathbb{E}_{S \sim \mathbb{P}^m} \big \{V(S) \big \} \leq \int_0^1 \epsilon q(\epsilon)~\mathrm{d}\epsilon. \label{eq:exp}
\end{equation}
The binomial tail $1 - \sum_{i = 0}^{r+d-1} \binom{m}{i} \epsilon^i (1-\epsilon)^{m-i}$ coincides with the cumulative distribution function of a beta distribution with parameters $r+d$ and $m+1-r-d$. Its density function is then given (can be also verified by differentiating the binomial tail with respect to $\epsilon$), by
\begin{align}
q(\epsilon) = (r+d) \binom{m}{r+d}\epsilon^{r+d-1}(1-\epsilon)^{m-r-d}.
\end{align}
Substituting the aforementioned expression of $q(\epsilon)$ in \eqref{eq:exp}, and by repeated integration by parts, it can be shown that the right-hand side in \eqref{eq:exp} is equal to $\frac{r+d}{m+1}$, thus concluding the proof.
\end{proof}

Note that this result would also hold with equality if the scenario program is such that all interim solutions violate the discarded samples. For the case where $\mathcal{R} = \emptyset$ (i.e., $r=0$), Theorem \ref{thm:scenario-approach-expectation-violation} reduces to 
one of the first scenario approach results that have appeared in the literature  \cite{Calafiore2004}. 
Theorem \ref{thm:scenario-approach-expectation-violation} generalizes that result to the case where samples are discarded. The only modification is that $d$ is replaced with $r+d$. This is since, under the discarding mechanism of the theorem, the sharp characterization of the cumulative distribution of Theorem \ref{thm:scenario-approach-discarded-constraints} established in \cite{Romao2023} is employed. For more generic discarding mechanisms, the bound in Theorem \ref{thm:scenario-approach-discarded-constraints} should be replaced with
$1 - \binom{r+d-1}{r} \sum_{i = 0}^{r+d-1} \binom{m}{i} \epsilon^i (1-\epsilon)^{m-i}$ (see \cite{Campi2010}). Due to the presence of an additional factor in front of the summation, the bound in Theorem \ref{thm:scenario-approach-expectation-violation} would then be $\binom{r+d-1}{r} \frac{r+d}{m+1}$.

\section{From conformal prediction to the scenario approach} \label{sec:conf_scen} 
In this section we show how to start from the original conformal prediction formulation developed in \cite{Vovk2005,Vovk2013} to derive a scenario approach bound on the expected probability of constrained violation for the optimal solution of an arbitrary convex scenario programs. To achieve this, we choose an appropriate nonconformity score that in turn allows us to establish connections between the associated set predictor and the set of samples for which constraints are violated. 

For the developments of this section consider the convex scenario program in \eqref{eq:scenario-program} with no discarded samples, i.e., $\mathcal{R}=\emptyset$ and hence $r=0$. Recall that the, unique under Assumption \ref{ass:unique}, optimal solution of \eqref{eq:scenario-program}, is denoted by $x^\star(S)$. We further assume in this section that the scenario program under consideration is fully-supported in the sense of Definition \ref{def:fully-supported-scenario-program}. This assumption could be relaxed by considering a regularized version of the problem as in \cite{Romao2023}; however, we do not pursue this discussion here.

Let $m\geq d$. We consider the following nonconformity measure $\mathcal{A}: \Omega^m \mapsto \mathbb{R}^m$ defined as
\begin{align}
        \mathcal{A}(S) = \Big [g(x^\star(S), \omega_1), \ldots, g(x^\star(S), \omega_m) \Big ]^\top, \label{eq:noncon}
\end{align}
where each element of the mapping is a valuation of the constraint function on $x^\star(S)$ and a sample $\omega_i$, $i=1,\ldots,m$. It can be observed that this mapping is permutation invariant. Under this choice, the nonconformity scores (see Section \ref{sec:prelim_conf}) are given by $R_i = g(x^\star(S), \omega_i)$, $i=1,\ldots,m$. If another sample $\omega$ is considered, the mapping would be $\mathcal{A}(S \cup \{\omega\}): \Omega^{m+1} \mapsto \mathbb{R}^{m+1}$. Its last element would be $R = g(x^\star(S\cup\{\omega\}), \omega)$, where $x^\star(S\cup\{\omega\})$ denotes the optimal solution of the scenario program when fed with the $m+1$ samples in  $S\cup \{\omega\}$.

Now, let $\delta = \frac{d}{m+1}$ and following \eqref{eq:Vanilla-CP} notice that the corresponding set predictor takes the form
\begin{equation}
        \Gamma_\delta(S) = \left\{ \omega \in \Omega: f(\omega; S) > \frac{d}{m+1}\right\},
        \label{eq:vanilla-CP-SP-connection}
\end{equation}
where $f(\omega; S)$ is defined as in \eqref{eq:func_f}. At the same time, consider the samples $\omega \in \Omega$, for which the optimal solution $x^\star(S)$ of \eqref{eq:scenario-program} remains feasible, namely,
\begin{equation}
        U = \big \{\omega \in \Omega: g(x^\star(S), \omega) \leq 0 \big \},
        \label{eq:SP-violation-probability-set}
\end{equation} 
We show that the set predictor in \eqref{eq:vanilla-CP-SP-connection} coincides with the set in \eqref{eq:SP-violation-probability-set}. This is summarized in the following proposition.

\begin{proposition}
Consider Assumptions \ref{ass:order} and \ref{ass:unique}, and further assume that the scenario program in \eqref{eq:scenario-program} is fully supported. Fix $\delta = \frac{d}{m+1}$. We have that that with $\mathbb{P}^m$-probability one with respect to the choice of $S$, 
\begin{align}
\Gamma_\delta(S) = U.
\end{align}
        \label{prop:CP-SP-connection}
        \hfill $\square$
\end{proposition}

\begin{proof}
        1. $U \subseteq \Gamma_\delta(S)$: Pick any $\omega \in U$. Since $g(x^\star(S), \omega) \leq 0$, we have that the optimal solution remains feasible when the constraints are enforced at this $\omega$. Therefore, solving the same scenario program with $m+1$ samples, namely, the ones in $S$ and $\omega$, will not change the optimal solution. We can formalize this as $x^\star(S \cup \{\omega\}) = x^\star(S)$, where $x^\star(S \cup \{\omega\})$ corresponds to the optimal solution when the constraints are enforced on the samples in $S \cup \{\omega\}$.
    
 Due to the assumption that the scenario program under consideration is fully supported with $\mathbb{P}^m$-probability one with respect to the choice of $S$, in the problem with $S$ samples exactly $d$ of them constituted the support set, which implies that exactly $d$ of them where active at the optimum. As such, in the problem with $S \cup \{\omega\}$, $\omega$ would result in an inactive constraint, as otherwise we would have $d+1$ samples in the support set.
 Therefore, there would be at least $d$ samples in $S$ such that $R_i=g(x^\star(S),\omega_i) \geq g(x^\star(S), \omega) = R$ (the $d$ ones that would be active would definitely have a higher constraint value).
 
 Therefore, we have that 
 \begin{align}
\big |\{i \in \{1,\ldots,m+1\}\}: R_i \geq R \big | > d, 
\end{align}
where the cardinality in the left-hand side would involve these $d$ samples and the sample leading to $R$, i.e., it would be at least $d+1$.
By \eqref{eq:func_f}, we thus have that $f(\omega; S) > \frac{d}{m+1}$, and as a result $\omega \in \Gamma_\delta(S)$. This shows that $U \subseteq \Gamma_\delta(S)$, thus concluding the first part of the proof.

 2. $\Gamma_\delta(S) \subseteq U$: This is equivalent to showing that $U^c \subset \Gamma_\delta(S)^c$, where the superscript $c$ denotes the set complement. Pick any $\omega \in U^c$. Since $g(x^\star(S), \omega) >0$ this then implies that the the solution $x^\star(S)$ becomes infeasible when it comes into the new sample $\omega$. Equivalently, when considering a scenario program with samples $S \cup \{\omega\}$ then $\omega$ would correspond to a constraint that is necessarily active at the new optimal solution. Since the program is assumed to be fully supported, at most $d-1$ of the samples $\omega_1,\ldots,\omega_m$ would result in active constraints for that problem. Therefore, 
  \begin{align}
\big |\{i \in \{1,\ldots,m+1\}\}: R_i \geq R \big | \leq d, 
\end{align}
where the cardinality in the left-hand side involves these $d-1$ samples plus the one leading to $R$.
By \eqref{eq:func_f}, this then implies that $\omega \in \Gamma_\delta(S)^c$, thus concluding the second part of the proof.   
\end{proof}

We can now combine Theorem \ref{thm:Vanilla-CP} -- a result from conformal prediction -- with Proposition \ref{prop:CP-SP-connection}, to recover Theorem \ref{thm:scenario-approach-expectation-violation} with $r=0$. This connection is formalized in the following theorem. 

\begin{thm}
Consider Assumptions \ref{ass:order} and \ref{ass:unique}, and further assume that the scenario program in \eqref{eq:scenario-program} with $r=0$ is fully supported. Consider also the nonconformity measure defined in \eqref{eq:noncon}. We then have that
   \begin{align}
                \mathbb{P}^{m+1} &\big \{(\omega_1,\ldots,\omega_m,\omega) \in \Omega^{m+1}: g(x^*(S),\omega) >0 \big \} \nonumber \\
                &=\mathbb{E}_{S\sim \mathbb{P}^m} \big \{ V(S) \big \} \leq \frac{d}{m+1}, \label{eq:thm4}
        \end{align}
        where $V(S)$ is as in \eqref{eq:violation-probability}.
        \label{thm:CP-SP-connection-average-violation}
        \hfill $\square$
\end{thm}

\begin{proof}
Consider the nonconformity measure defined in \eqref{eq:noncon}. Set then $\delta = \frac{d}{m+1}$ and consider the set $\Gamma_\delta(S)$ in \eqref{eq:vanilla-CP-SP-connection}.
By Proposition \ref{prop:CP-SP-connection}, the left-hand side in \eqref{eq:thm4} is equal to 
\begin{align}
\mathbb{P}^{m+1} \big \{(\omega_1,\ldots,\omega_m,\omega) \in \Omega^{m+1}: \omega \notin \Gamma_\delta(S) \big \}, 
\end{align}
which by means of Theorem \ref{thm:Vanilla-CP} and our choice of $\delta$, is upper-bounded by $\frac{d}{m+1}$. It remains to show the equality in \eqref{eq:thm4}. This follows from \cite[Proposition 3]{Campi_TAC2009note}; see also the proof of Proposition \ref{prop:exp_int}, where this argument is shown.
\end{proof}
By the discussion below Theorem \ref{thm:scenario-approach-expectation-violation}, for fully-supported programs we would expect Theorem \ref{thm:CP-SP-connection-average-violation} to be in fact tight.


\section{From the scenario approach to conformal prediction} \label{sec:scen_conf}
In this section we follow the opposite route, and show how using tools from scenario optimization we can recover conformal prediction bounds, offering a different interpretation.

\subsection{Towards vanilla conformal prediction}  \label{sec:scen_conf_van}
Consider a set of $m$ independent samples $S = \{\omega_1,\ldots,\omega_m\}$ generated according to $\mathbb{P}$. In view of establishing a link with conformal prediction, assume that these samples are mapped to another collection of $m$ scalar samples $R_1,\ldots,R_m$, that will in turn play the role of nonconformity scores, as reasoned in Section \ref{sec:prelim_conf}. We consider Assumption \ref{ass:order}, and assume here that these samples are in ascending order without any repetitions, i.e.,
\begin{align}
R_1 < R_2 < \ldots < R_m.
\end{align}

Fix $r \in [0,m]$, and consider the following discarding procedure: the highest sample among $R_1,\ldots,R_m$ is discarded; once the remaining samples are identified, we repeat this process until $r$ samples are discarded in total. 
We then consider the following optimization problem
\begin{align}
        \minimize_{\overline{R} \in \mathbb{R}} & \quad \overline{R} \nonumber \\
        \st & \quad R_i \leq \overline{R}, \quad \forall i = 1,\ldots,m-r.
        \label{eq:scenario-program-order}
\end{align}
This problem is an one-dimensional ($d=1$, as the only decision variable is $\overline{R}$) convex scenario program. Notice that it is directly in the form of \eqref{eq:scenario-program}, with the $S\setminus \mathcal{R}$ being the $R_1,\ldots, R_{m-r}$, i.e., the samples remaining upon discarding $r$ of them. This program encodes the problem of determining the $(m-r)$-smallest sample in a sorted list. It is pictorially illustrated in Figure \ref{fig:samples_order}.

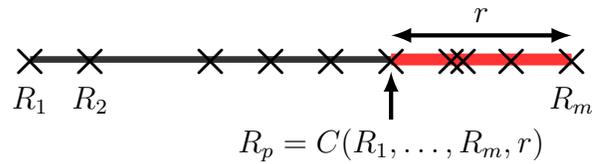
\begin{figure}[t!]
\centering
\begin{tikzpicture}[>=latex, font=\large, scale = 1.6]   
\draw[black!80, line width=2.5pt] plot[domain=-4.5:-1.5] (\x,{0});   
\draw[red!80, line width=4.5pt] plot[domain=-1.5:0] (\x,{0});   
\draw[<->,ultra thick] (-1.5,0.2)--(0,0.2) node[left]{};
\draw[] (-4.5,0.21) node[below, bag]{$\LARGE{\boldsymbol{\times}}$};
\draw[] (-4,0.21) node[below, bag]{$\LARGE{\boldsymbol{\times}}$};
\draw[] (-3,0.21) node[below, bag]{$\LARGE{\boldsymbol{\times}}$};
\draw[] (-2.5,0.21) node[below, bag]{$\LARGE{\boldsymbol{\times}}$};
\draw[] (-2,0.21) node[below, bag]{$\LARGE{\boldsymbol{\times}}$};
\draw[] (-1.5,0.21) node[below, bag]{$\LARGE{\boldsymbol{\times}}$};
\draw[] (-1,0.21) node[below, bag]{$\LARGE{\boldsymbol{\times}}$};
\draw[] (-0.9,0.21) node[below, bag]{$\LARGE{\boldsymbol{\times}}$};
\draw[] (-0.5,0.21) node[below, bag]{$\LARGE{\boldsymbol{\times}}$};
\draw[] (0,0.21) node[below, bag]{$\LARGE{\boldsymbol{\times}}$};
\draw[] (-4.5,-0.15) node[below, bag]{$R_1$};
\draw[] (-4,-0.15) node[below, bag]{$R_2$};
\draw[] (0,-0.15) node[below, bag]{$R_m$};
\draw[] (-0.75,0.5) node[below, bag]{$r$};
\draw[->,ultra thick] (-1.5,-0.5)--(-1.5,-0.1) node[left]{};
\draw[] (-1.5,-0.5) node[below, bag]{$R_p=C(R_1,\ldots,R_m,r)$};
\end{tikzpicture}
\caption{Scenario program pictorial construction: Given a collection of ordered samples (nonconformity scores), $R_1,\ldots,R_m$, discard the $r$ highest ones (the ones that fall in the red region). The maximum out of the remaining ones is the solution of an one-dimensional scenario program, denoted by $R_p=C(R_1,\ldots,R_m,r)$.}
\label{fig:samples_order}
\end{figure}

Let $p=m-r$. The optimal solution of the aforementioned problem is $R_p := C(R_1,\ldots,R_m,r)$. By the function $C$ we emphasize the dependence of $R_p$ on all samples (even the ones that are discarded), and the number of discarded samples $r$. The probability of constraint violation for this problem would be of the form
\begin{align}
V(S) = \mathbb{P} \big \{ R> R_p \big \}.  \label{eq:prob_viol_1d}
\end{align}
Notice the slight abuse of notation: the probability here refers to selecting a new $\omega$, which will induce a new $R$.

Under Assumption \ref{ass:order}, \eqref{eq:scenario-program-order} is a fully-supported problem in the sense of Definition \ref{def:fully-supported-scenario-program}. Moreover, since $d=1$, our discarding procedure coincides with the one stated in Theorem \ref{thm:scenario-approach-discarded-constraints}. In particular, once a sample is removed, the interim maximum sample is strictly smaller than the removed one, or in other words the discarded samples are violated by all interim solutions. As argued below Theorem \ref{thm:scenario-approach-discarded-constraints}, the results of Theorems \ref{thm:scenario-approach-discarded-constraints} and \ref{thm:scenario-approach-expectation-violation} would thus hold with equality.

By Theorem \ref{thm:scenario-approach-expectation-violation}, and upon substituting the expression for the probability of constraint violation for our problem, 
\begin{equation} \label{eq:exp_1d}
\mathbb{E}_{S \sim \mathbb{P}^m} \big \{ \mathbb{P} \big \{ R> R_p \big \} \big \} = \frac{r+1}{m+1}.
\end{equation}
Notice that the numerator in the right-hand side of \eqref{eq:exp_1d} is $r+1$ as $d=1$.
We have the following proposition due to \cite[Proposition 3]{Campi_TAC2009note}; we provide a proof here for completeness, adapted to our context.

\begin{proposition} \label{prop:exp_int}
We have that 
\begin{align}
&\mathbb{E}_{S \sim \mathbb{P}^m} \big \{ \mathbb{P} \big \{ R> R_p \big \} \big \} \nonumber \\
&~~= \mathbb{P}^{m+1} \big \{ (\omega_1,\ldots,\omega_m,\omega) \in \Omega^{m+1}:~ R> R_p \big \}.
\end{align}
\hfill $\square$
\end{proposition}

\begin{proof}
We have that 
\begin{align}
&\mathbb{E}_{S \sim \mathbb{P}^m} \big \{ \mathbb{P} \big \{ R> R_p \big \} \big \} \nonumber \\
&~~= \int_{\Omega^m} \Bigg ( \int_{\Omega} \mathbf{1}_{\big \{R>R_p \big \}}~ \mathrm{d} \mathbb{P}(\omega) \Bigg ) \mathrm{d} \mathbb{P}^m(\omega_1,\ldots,\omega_m) \nonumber \\
&~~= \int_{\Omega^{m+1}} \mathbf{1}_{\big \{R>R_p \big \}}~ \mathrm{d} \mathbb{P}^{m+1}(\omega_1,\ldots,\omega_m,\omega) \nonumber \\
&~~= \mathbb{P}^{m+1} \big \{ (\omega_1,\ldots,\omega_m,\omega) \in \Omega^{m+1}:~ R> R_p \big \}, 
\end{align}
where the first equality is by the definition of the expected value (outer integral), and the definition of the probability as an integral (term in parenthesis, where the argument of that integral is an indicator function). The second equality involves combining the integrals (recall that samples are independent), and the last one is by reverting from an $m+1$-dimensional integral to the associated $\mathbb{P}^{m+1}$ measure. 
\end{proof}

Proposition \ref{prop:exp_int} establishes the fact that the expected value of the probability of constraint violation is identical to the $\mathbb{P}^{m+1}$ measure of the associated event, which offers an assessment based on one level of probability.
Combining \eqref{eq:exp_1d} and Proposition \ref{prop:exp_int}, leads to the following theorem.
\begin{thm} \label{thm:scen_conf_van}
Consider Assumptions \ref{ass:order} and \ref{ass:unique}.
Fix $\delta \in (0,1)$, and consider the scenario program in \eqref{eq:scenario-program-order} with $r = m -\nceil{(1-\delta)(m+1)}$. We then have that $p = \nceil{(1-\delta)(m+1)}$, and 
\begin{align}
\mathbb{P}^{m+1} \big \{ (\omega_1,\ldots,\omega_m,\omega) \in \Omega^{m+1}:~ R> R_p \big \} \leq \delta.
\end{align}
\hfill $\square$
\end{thm}

\begin{proof}
By Proposition \ref{prop:exp_int} and \eqref{eq:exp_1d}, we have that
\begin{align}
\mathbb{P}^{m+1} &\big \{ (\omega_1,\ldots,\omega_m,\omega) \in \Omega^{m+1}:~ R> R_p \big \} = \frac{r+1}{m+1} \nonumber \\
& = \frac{m+1 - \nceil{(1-\delta)(m+1)}}{m+1} \nonumber \\
&= \frac{\nfloor{\delta(m+1)}}{m+1} \leq \delta.
\end{align}
where the second equality is due to the choice of $r$, and the third is since $m+1 - \nceil{(1-\delta)(m+1)} = \nfloor{\delta(m+1)}$. The last inequality is since $\nfloor{\delta(m+1)} \leq \delta (m+1)$, and thus concludes the proof.
\end{proof}

The statement of Theorem \ref{thm:scen_conf_van}  would be tight, i.e., it would hold with equality, if $\nfloor{\delta(m+1)}$ is an integer.

Theorem \ref{thm:scen_conf_van} shows that $R_p = C(R_1,\ldots,R_m,r)$ can be interpreted as the $1-\delta$ quantile $\mathrm{Quantile}_{1-\delta}(R_1,\ldots,R_m,\infty)$
over the nonconformity scores, as under the choice of $r$ there would exist $\nceil{(1-\delta)(m+1)}$ nonconformity scores lower than or equal to $R_p$. It can be in turn computed via the one-dimensional scenario program in \eqref{eq:scenario-program-order} with an appropriate number of discarded samples. Leveraging then bounds of the scenario approach on the expected probability of constraint violation we recover the result of Corollary \ref{cor:viol}.

\begin{remark} \label{rem:sample1}
Theorem \ref{thm:scen_conf_van} links three quantities: the total number of samples $m$, the number of samples to be discarded $r$, and the probability/confidence level $\delta$. We need to choose two out of the three parameters and infer the other one. In the theorem's statement we fixed $m$ and $\delta$, while $r$ was a function of them. An alternative use would be to choose $\delta$ and $r$ (the number of samples we are willing to discard), and infer the number of samples/data we need. Therefore,
\begin{align}
\frac{r+1}{m+1} \leq \delta ~\Longleftrightarrow~ m \geq \frac{r+1}{\delta} - 1,
\end{align}
and serves as an explicit sample size bound to meet the desired confidence level $\delta$.
\end{remark}

\subsection{Towards calibration conditional conformal prediction}  \label{sec:scen_conf_cond}
Consider \eqref{eq:scenario-program-order} and recall that, as discussed in the previous subsection, under Assumption \ref{ass:order}, it is a fully-supported problem. Theorem \ref{thm:scenario-approach-discarded-constraints} is then applicable, while since it is an one-dimensional problem ($d=1$) and due to the specific discarding procedure, the statement of Theorem \ref{thm:scenario-approach-discarded-constraints} is tight, i.e., it holds with equality.
We then have the following theorem, directly emanating from Theorem \ref{thm:scenario-approach-discarded-constraints}, under the constraint violation probability in \eqref{eq:prob_viol_1d}.
 \begin{thm}
Consider Assumptions \ref{ass:order} and \ref{ass:unique}.
Fix $\epsilon \in (0,1)$, and consider the scenario program in \eqref{eq:scenario-program-order} with $r = m -\nceil{(1-\epsilon)(m+1)}$. Set then $\delta \in (0,1)$ such that
\begin{align}
\delta = \sum_{i = 0}^{r} \binom{m}{i} \epsilon^i (1-\epsilon)^{m-i}.
\end{align}
We then have that $p = \nceil{(1-\epsilon)(m+1)}$, and 
        \begin{align}
                \mathbb{P}^m \big \{ S \in \Omega^m&: \mathbb{P} \big \{ R> R_p \big \} \leq \epsilon \big \} = 1-\delta. \label{eq:bound}
        \end{align}
        \label{thm:scen_conf_cond}
        \hfill $\square$
\end{thm}
Notice that the upper limit in the summation in the right-hand side of \eqref{eq:bound} is $r$ as opposed to $r+d-1$ in Theorem \ref{thm:scenario-approach-discarded-constraints}, as $d=1$ for the problem in \eqref{eq:scenario-program-order}. Moreover, the result is tight, as argued prior to the theorem statement.

Theorem \ref{thm:scen_conf_cond} shows that $R_p = C(R_1,\ldots,R_m,r)$ can be interpreted as the $1-\epsilon$ quantile $\mathrm{Quantile}_{1-\epsilon}(R_1,\ldots,R_m,\infty)$. However, this time we have a two-layer of probability assessment. The inner probability is a random variable depending on the so called calibration samples $R_1,\ldots,R_m$, and hence the term calibration conditional conformal prediction. 

\begin{remark}\label{rem:sample2}
In general, fixing any $\epsilon \in (0,1)$ and an integer $r$ for the samples to be discarded, we have that for $d=1$,
        \begin{align}
                \mathbb{P}^m \big \{ S \in \Omega^m&: \mathbb{P} \big \{ R> R_p \big \} \leq \epsilon \big \} \nonumber \\&= 1-\sum_{i = 0}^{r} \binom{m}{i} \epsilon^i (1-\epsilon)^{m-i}.
        \end{align}
        Notice that this results holds with equality due to the particular structure of the scenario program and the discarding mechanism. As such, this directly implies that $\mathbb{P} \big \{ R> R_p \big \}$ is distributed according to the tail of 
        binomial distribution, which in turn coincides with 
        \begin{align}
        \mathbb{P} \big \{ R> R_p \big \} \sim \mathrm{Beta}(r+1,m-r)
        \end{align}
        where $\mathrm{Beta}$ denotes a beta distribution with parameters $r+1$ and $m-r$. This is in agreement with \cite{Vovk1999,Angelopoulos2023,Lindemann2024}.

        Alternatively, given $\epsilon, \delta \in (0,1)$, and $r$, one can determine the number of samples required so that \eqref{eq:bound} is satisfied. A sufficient but explicit condition for this to happen is if
        \begin{align}
        m \geq \frac{2}{\epsilon} \Bigg ( r+ \ln \frac{1}{\delta} \Bigg ).
        \end{align}
        This follows directly from \cite[Section 3.2.1]{Campi2018} (due to the particular scenario program under consideration that allows for the bound in Theorem \ref{thm:scenario-approach-discarded-constraints}). In the expression in \cite[Section 3.2.1]{Campi2018} we set $r+d$ in place of $d$ as we now have discarded samples, and substitute for $d=1$. 
        \end{remark}
        
Calibration conditional conformal prediction bounds were revisited in \cite[Section 2.2]{Lindemann2024}. Adapting that result to our notation, this would require fixing $\epsilon, \delta \in (0,1)$, and the total number of samples $m$, and choosing 
 \begin{align}
 r = m - \Big \lceil \Big (1-\epsilon + \sqrt{\frac{\ln (1/\delta)}{2m}} \Big )(m+1) \Big \rceil.
 \end{align}
Under this choice, and since $p = m-r$, the constructed quantile is the $\Big (1-\epsilon + \sqrt{\frac{\ln (1/\delta)}{2m}} \Big )$ one, as opposed to the $1-\epsilon$ quantile of Theorem \ref{thm:scen_conf_cond}.

It is worth mentioning that the approach in \cite{Lin2024} attempted to connect the calibration conditional conformal prediction and the scenario approach with sample discarding. The analysis therein also employed a scenario approach with sample discarding bound which, requires (as with our work) the number of sampled to 
 to be discarded to be \emph{a priori} fixed. However, 
in \cite{Lin2024} this number was determined once all samples are observed and is hence an \emph{a posteriori} quantity, thus hampering the ability to employ such bound. Here we use discarding via a completely different analysis technique, thus not being affected by the aforementioned issue.

\section{Concluding remarks and future work} \label{sec:conc}

This paper investigates fundamental connections between conformal prediction and the scenario approach. We have elicited specific choices of score functions and set predictor for which we can study properties of the optimal solution of scenario programs using conformal prediction results. We have also drew connections on the opposite direction, namely, using scenario approach results to study guarantees on the validity of the predictor set.

Current research concentrates towards extending Proposition \ref{prop:CP-SP-connection} to scenario programs that are not necessarily fully-supported, using a regularization procedure as in \cite{Romao2023}. Moreover, we aim at employing adaptive conformal prediction concepts in scenario optimization, and using recent \emph{a posteriori} results on the scenario approach can be used to guide new conformal prediction methods. 

\addtolength{\textheight}{-12cm}   






\bibliographystyle{IEEEtran}
\bibliography{cdc-bib-files.bib, cdc-bib-files-2.bib}

\end{document}